\documentclass{article}
\usepackage{iclr2026_conference,times}

\usepackage{amsmath}
\usepackage{amssymb}
\usepackage{amsthm}
\newtheorem{theorem}{Theorem}

\usepackage{graphicx}
\graphicspath{{}}
\usepackage{booktabs}
\usepackage{hyperref}
\usepackage{url}
\usepackage{tikz}
\usetikzlibrary{arrows.meta,positioning,calc}

\usepackage{amsmath,amsfonts,bm}

\def\eqref#1{equation~\ref{#1}}

\def\1{\bm{1}}

\DeclareMathAlphabet{\mathsfit}{\encodingdefault}{\sfdefault}{m}{sl}
\SetMathAlphabet{\mathsfit}{bold}{\encodingdefault}{\sfdefault}{bx}{n}

\title{Structured Abductive-Deductive-Inductive\\Reasoning for LLMs via Algebraic Invariants}

\author{
Sankalp Gilda\thanks{Corresponding author.} \\
DeepThought Solutions \\
\texttt{sankalp.gilda@gmail.com} \\
\And
Shlok Gilda \\
Department of Computer Science \\
University of Florida \\
\texttt{shlok.gilda@ufl.edu}
}

\iclrfinalcopy 

\begin{document}
\raggedbottom

\maketitle

\begin{abstract}
Large language models exhibit systematic limitations in structured logical reasoning: they conflate hypothesis generation with verification, cannot distinguish conjecture from validated knowledge, and allow weak reasoning steps to propagate unchecked through inference chains. We present a symbolic reasoning scaffold that operationalizes Peirce's tripartite inference---abduction, deduction, and induction---as an explicit protocol for LLM-assisted reasoning. The framework enforces logical consistency through five algebraic invariants (the \emph{Gamma Quintet}), the strongest of which---the Weakest Link bound---ensures that no conclusion in a reasoning chain can exceed the reliability of its least-supported premise. This principle, independently grounded as weakest link resolution in possibilistic logic and empirically validated for chain-of-thought reasoning, prevents logical inconsistencies from accumulating across multi-step inference. We verify all invariants through a property-based testing suite of 100 properties and 16 fuzz tests over $10^5$+ generated cases, providing a verified reference implementation of the invariants suitable as a foundation for future reasoning benchmarks.
\end{abstract}

\section{Introduction}\label{sec:introduction}

Large language models have achieved strong performance on reasoning tasks, yet they exhibit systematic failures in structured logical reasoning. On combinatorial logic puzzles, accuracy degrades sharply as problem complexity increases---a phenomenon termed the ``curse of complexity''~\citep{zebralogic2025}---revealing that LLMs struggle to maintain logical consistency across multi-step inference. More fundamentally, chain-of-thought explanations are only 25--39\% faithful to the model's actual computation~\citep{anthropic2025faithfulness}, meaning the reasoning traces that users rely upon to assess answer quality frequently do not reflect the process that produced the answer.

These failures stem from a structural conflation of three distinct inference modes that have been recognized since Peirce's foundational work on the logic of science~\citep{peirce1878deduction}:

\begin{itemize}
\item \textbf{Abduction} (hypothesis generation): generating candidate explanations for observed phenomena.
\item \textbf{Deduction} (logical verification): deriving necessary consequences from premises.
\item \textbf{Induction} (empirical validation): testing predictions against observations.
\end{itemize}

\noindent In a single autoregressive pass, LLMs perform all three simultaneously: generating hypotheses, checking them against implicit constraints, and marshaling evidence---without marking which mode is active at any step. Chain-of-thought prompting~\citep{wei2022chainofthought} approximates deduction but provides no formal guarantees. Self-consistency voting~\citep{wang2023selfconsistency} approximates induction but averages over candidates rather than validating them. Process reward models~\citep{lightman2024verify} score individual steps but do not enforce structural constraints across reasoning chains. None of these approaches explicitly separate inference modes, formally bound how weak premises propagate, or maintain a persistent, auditable knowledge state across interactions.

We present a symbolic reasoning scaffold with three contributions mapped to the workshop's topics of interest:

\begin{enumerate}
\item \textbf{An explicit ADI protocol} separating abduction, deduction, and induction into distinct, auditable phases with epistemic state tracking (Topic~1: deduction, induction, and abduction).

\item \textbf{Five algebraic invariants} (the Gamma Quintet) that formally constrain how reliability propagates through reasoning chains, with the Weakest Link bound preventing logical inconsistencies from accumulating (Topics~2 and~3: symbolic reasoning and avoiding contradictions).

\item \textbf{A property-based verification benchmark} of 100 test properties and 16 fuzz tests validating that implementations preserve these invariants across $10^5$+ randomly generated cases (Topic~5: benchmarks and evaluation).
\end{enumerate}

\noindent The framework operates as an external symbolic system alongside the LLM (Topic~4: external logical solvers), maintaining a knowledge graph with formal consistency guarantees while the LLM handles natural language reasoning.

The Weakest Link bound---ensuring that no aggregated conclusion exceeds the reliability of its least-supported input---has independent theoretical grounding in possibilistic logic~\citep{dubois2025possibilistic}, where it is known as ``weakest link resolution,'' and recent empirical validation demonstrating that chain-of-thought reliability is bounded by its weakest step~\citep{jacovi2024weakestlink}. This convergence from algebraic specification, possibility theory, and empirical measurement provides triangulated justification for the framework's central consistency constraint.

\section{Symbolic Knowledge Representation}\label{sec:fgr}

We represent knowledge claims as structured symbolic objects carrying a three-dimensional descriptor: \textbf{Formality~(F)}, measuring rigor of expression; \textbf{Scope~(G)}, bounding the context where the claim applies; and \textbf{Reliability~(R)}, a computed consistency score on $[0, 1]$. These descriptors let the framework track the epistemic status of LLM-generated conclusions.

\begin{table}[ht]
\caption{Formality levels and reliability ceilings.}
\label{tab:formality}
\centering
\small
\begin{tabular}{@{}llr@{}}
\toprule
\textbf{Level} & \textbf{Description} & \textbf{Ceiling} \\
\midrule
F0 & Informal (anecdotal, authority-based) & 70\% \\
F1 & Structured (ADRs, explicit trade-offs) & 85\% \\
F2 & Empirical (benchmarks, load tests) & 95\% \\
F3 & Formal (proofs, model checking, type-checked) & 99\% \\
\bottomrule
\end{tabular}
\end{table}

\noindent F3 admits machine-checkable verification; recent work~\citep{typedcot2026} proposes extending this to LLM reasoning via the Curry-Howard correspondence, treating reasoning traces as type-checkable proof terms. The F3 ceiling is 99\% rather than 100\%, reflecting that formal proofs depend on unverified proof checkers~\citep{pollack1998verify}. The ordering invariant $C_{F_0} < C_{F_1} < C_{F_2} < C_{F_3}$ ensures that higher formality always permits higher reliability. Orthogonal epistemic layers mark symbolic state: L0 (conjecture, 35\%), L1 (substantiated, 75\%), L2 (corroborated, 100\%), corresponding to the abductive, deductive, and inductive phases respectively.

\textbf{Faithfulness ceiling.} Chain-of-thought explanations are only 25--39\% faithful to the model's actual computation~\citep{anthropic2025faithfulness}, suggesting LLM-generated evidence should cap at F1 with faithfulness as a limiting factor: applying the framework's own min-aggregation principle, the effective ceiling is $\min(0.85, 0.39) = 0.39$ for current models.

The complete effective reliability formula combines all constraints via consistency-preserving inference:
\begin{equation}\label{eq:reff}
R_{\text{eff}} = \min\bigl(
  \min_i R_{\text{adj}}(e_i),\;
  \min_j \max\bigl(0,\, R_{\text{eff}}(d_j) - \text{CL}_j\bigr),\;
  C_L,\;
  C_F
\bigr)
\end{equation}
where $R_{\text{adj}}(e_i)$ is the adjusted score for evidence~$i$ (including temporal decay), $\text{CL}_j$ is the congruence penalty for dependency~$j$, and $C_L$, $C_F$ are the layer and formality ceilings respectively. The $\max(0, \cdot)$ floor ensures all terms stay within $[0, 1]$ even when a congruence penalty exceeds the incoming reliability. The nested $\min$ structure ensures that no individual component can inflate the aggregate---a property we formalize as the WLNK invariant in Section~\ref{sec:invariants}.

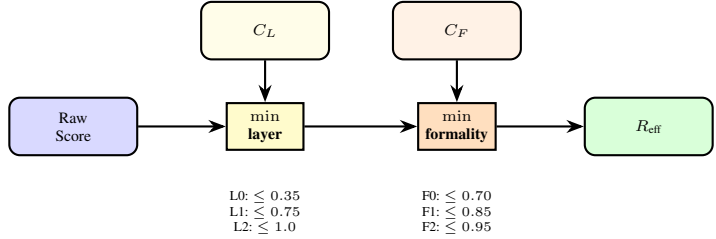
\begin{figure}[ht]
\centering
\begin{tikzpicture}[scale=0.85, transform shape,
    >=Stealth,
    block/.style={draw=black,thick,rectangle,rounded corners,minimum width=2.0cm,minimum height=0.9cm,align=center,font=\scriptsize},
    gate/.style={draw=black,thick,rectangle,minimum width=1.2cm,minimum height=0.7cm,align=center,font=\scriptsize\bfseries},
    arrow/.style={->,>=Stealth,thick}
]
\node[block,fill=blue!15] (raw) at (0,0) {Raw\\Score};
\node[gate,fill=yellow!25] (lg) at (3,0) {$\min$\\layer};
\node[gate,fill=orange!25] (fg) at (6,0) {$\min$\\formality};
\node[block,fill=green!15] (out) at (9,0) {$R_{\text{eff}}$};

\node[block,fill=yellow!10] (lc) at (3,1.5) {$C_L$};
\node[block,fill=orange!10] (fc) at (6,1.5) {$C_F$};

\draw[arrow] (raw) -- (lg);
\draw[arrow] (lg) -- (fg);
\draw[arrow] (fg) -- (out);
\draw[arrow] (lc) -- (lg);
\draw[arrow] (fc) -- (fg);

\node[below=0.5cm of lg,font=\tiny,align=center] {L0: $\leq 0.35$\\L1: $\leq 0.75$\\L2: $\leq 1.0$};
\node[below=0.5cm of fg,font=\tiny,align=center] {F0: $\leq 0.70$\\F1: $\leq 0.85$\\F2: $\leq 0.95$};
\end{tikzpicture}
\caption{Dual ceiling constraint. The propagated reliability score passes through two successive $\min$ gates: the epistemic layer ceiling, then the formality ceiling. An L0 claim at F3 is capped at $\min(0.35, 1.0) = 0.35$. Because the two ceilings are independent, the tighter always dominates: L0 claims are layer-bounded regardless of formality, and fully corroborated L2 claims ($C_L = 1.0$) are formality-bounded regardless of layer. Both dimensions become active only at intermediate layers.}
\label{fig:dualceiling}
\end{figure}

Evidence transfers across contexts incur congruence penalties: CL3 (same context, no penalty), CL2 (similar, $-0.1$), CL1 (different, $-0.4$). Evidence with scope match ``none'' is excluded entirely.

\textbf{Verification credibility.} Evidence scores are adjusted by a verification method multiplier: self-reported ($\times 0.60$), script-attached ($\times 0.85$), externally reviewed ($\times 0.95$), and executed-and-verified ($\times 1.00$). This prevents self-reported evidence from achieving the same influence as independently verified results.

\section{The ADI Reasoning Protocol}\label{sec:adi}

Current LLM reasoning techniques treat inference as a monolithic process:
a single autoregressive pass produces hypotheses, justifications, and
conclusions without distinguishing the epistemic character of each step.
We propose the \emph{Abduction--Deduction--Induction} (ADI) protocol, which
decomposes LLM reasoning into three explicitly labeled inference modes,
each with distinct epistemic commitments and verification requirements.

\subsection{Peircean Inference and LLM Reasoning}\label{sec:peirce}

Charles Sanders Peirce classified inference into three irreducible
modes~\citep{peirce1878deduction,peirce1903harvard}: \emph{abduction}
(inference to the best explanation), \emph{deduction} (necessary inference
from premises), and \emph{induction} (inference from observed instances).
Peirce argued that these modes are distinct epistemic operations, each with its own
warrant structure and failure modes~\citep{magnani2009abductive}.

Contemporary LLM reasoning conflates all three in ways that obscure the
epistemic status of generated claims:
\begin{itemize}
    \item \textbf{Chain-of-thought prompting}~\citep{wei2022chainofthought}
        approximates deductive reasoning by eliciting step-by-step
        derivations, but provides no formal guarantee that steps follow
        from premises. The model may produce a sequence that \emph{reads}
        deductively while smuggling in abductive leaps---plausible
        hypotheses presented as logical necessities.
    \item \textbf{Self-consistency voting}~\citep{wang2023selfconsistency}
        approximates induction by sampling multiple reasoning paths and
        selecting the majority answer. However, it \emph{averages} over
        paths rather than \emph{validating} any single path against
        empirical evidence. Agreement among hallucinated chains does not
        constitute inductive support.
    \item \textbf{Neither technique explicitly marks which inference mode
        is active at each step.} A chain-of-thought trace may begin with
        an abductive hypothesis (``this likely fails because\ldots''),
        shift to deductive reasoning (``given X, it follows that\ldots''),
        and conclude with an inductive generalization (``in similar cases,
        we observe\ldots'')---all without the model or user being aware
        of the transitions.
\end{itemize}

This conflation has a direct analogue in dual-process theory~\citep{kahneman2011thinking}: abduction is fast and pattern-driven (System~1), while deduction and induction are deliberate and rule-governed (System~2). The ADI protocol makes this handoff explicit and auditable.

\subsection{The ADI Protocol}\label{sec:adi-protocol}

The protocol organizes reasoning as a cycle of three phases, each producing
claims at a successively higher epistemic layer. We describe each phase,
its epistemic commitments, and its realization in LLM-assisted reasoning.

\subsubsection{Abduction: Hypothesis Generation (L0---Conjecture)}

Given an anomaly, question, or design problem, the reasoning process begins
by generating candidate hypotheses. These are conjectures: plausible but
unverified explanations inferred from incomplete evidence. In the framework's
epistemic hierarchy, abductive claims are labeled L0 (conjecture) with
reliability capped at 35\%, reflecting their status as unvalidated proposals.

In LLM reasoning, abduction corresponds to the \emph{hypothesis generation
step}. The model draws on patterns in its training distribution to propose
explanations. Multiple candidates are not merely tolerated but encouraged:
the strength of abduction lies in generating a diverse candidate set, not
in premature commitment to a single explanation.

\paragraph{Example.} Consider an LLM-assisted analysis of a retrieval
system. The model proposes: ``Increasing the context window from 4K to 32K
tokens would improve retrieval accuracy, based on patterns showing
correlation between context length and answer quality in multi-document
settings.'' At this stage, the claim is plausible---it aligns with known
properties of attention mechanisms---but it is \emph{unverified}. It
carries an L0 label and cannot propagate as established knowledge.

\paragraph{Epistemic commitment.} Abduction makes no truth claim. It
asserts only that a hypothesis is \emph{worth investigating}. The failure
mode is not generating a wrong hypothesis (that is expected) but treating
an L0 conjecture as if it were established fact---precisely the failure
mode of unconstrained autoregressive generation.

\subsubsection{Deduction: Logical Verification (\texorpdfstring{L0$\to$L1}{L0 to L1}---Substantiation)}

Deductive verification checks whether a hypothesis is logically consistent
with the current body of validated knowledge. The question is not ``is this
true?'' but ``does this contradict anything we already know?'' A hypothesis
that survives deductive scrutiny is promoted to L1 (substantiated): it is
internally consistent but not yet empirically confirmed.

In LLM reasoning, deduction is where contradictions with prior validated
knowledge are detected. This phase serves as a logical filter, catching
hypotheses that are plausible in isolation but incompatible with established
constraints.

\paragraph{Critical structural requirement.} Deductive verification
requires an \emph{external} check. A system cannot reliably verify the
consistency of its own outputs against its own knowledge, because the same
biases that produced the hypothesis contaminate the verification.
Step-by-step verification~\citep{lightman2024verify} partially addresses
this by training separate verifier models, but the principle is more
general: the verifier must have access to a constraint set that is
independent of the generation process.

\paragraph{Example (continued).} The 32K context hypothesis is checked
against the validated knowledge base. Deductive analysis reveals: ``The
hypothesis is logically consistent with known attention mechanism
properties. However, it contradicts a previously validated finding (L2)
that retrieval accuracy plateaus beyond 16K tokens on our specific document
corpus due to the long-tail distribution of relevant passages.'' The
contradiction forces one of three outcomes: (a)~refine the hypothesis to
account for the constraint (e.g., restrict the claim to multi-hop queries),
(b)~challenge the prior L2 finding with new evidence, or (c)~discard the
hypothesis. The deductive phase does not resolve the question---it
sharpens it.

\paragraph{Epistemic commitment.} A claim at L1 is asserting: ``this is
logically consistent with everything we have validated so far.'' It is
\emph{not} asserting truth. The failure mode is passing a hypothesis that
contains a hidden contradiction---which is why the external verification
requirement is structural, not optional.

\subsubsection{Induction: Empirical Validation (\texorpdfstring{L1$\to$L2}{L1 to L2}---Corroboration)}

Inductive validation tests the hypothesis against empirical evidence.
An L1 claim that passes empirical validation is promoted to L2
(corroborated): it is both logically consistent and empirically supported,
within a specified scope.

In LLM reasoning, this is where claims are tested against reality---not
against the model's beliefs about reality. Induction requires running
experiments, collecting data, or testing predictions on held-out
observations.

\paragraph{Example (continued).} The refined hypothesis---``32K context
improves retrieval accuracy specifically for multi-hop queries''---is
tested on a held-out corpus. The benchmark confirms: retrieval F1 improves
from 0.72 to 0.81 for multi-hop queries, with no statistically significant
gain for single-hop queries. The hypothesis is empirically validated with
an explicit scope constraint: the L2 claim specifies the conditions under
which it holds.

\paragraph{Epistemic commitment.} A claim at L2 asserts: ``this has been
observed to hold under specified conditions.'' It does \emph{not} assert
universal truth. The scope constraint is essential: inductive claims
without scope constraints are unfalsifiable and therefore epistemically
vacuous. Every L2 claim in the framework carries an explicit validity
window and scope specification.

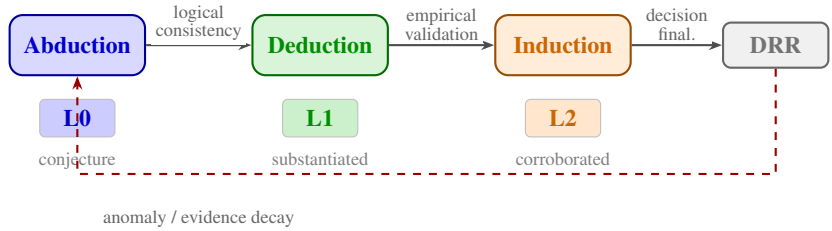
\begin{figure}[t]
\centering
\begin{tikzpicture}[
    node distance=1.6cm and 0.6cm,
    phase/.style={
        draw=#1!60!black, thick, rounded corners=5pt,
        minimum width=1.8cm, minimum height=0.8cm,
        font=\small\bfseries, text=#1!80!black,
        fill=#1!15
    },
    layer/.style={
        draw=gray!50, rounded corners=2pt,
        minimum width=1.0cm, minimum height=0.5cm,
        font=\footnotesize\bfseries, fill=#1!20, text=#1!80!black
    },
    drr/.style={
        draw=gray!80!black, thick, rounded corners=3pt,
        minimum width=1.4cm, minimum height=0.6cm,
        font=\small\bfseries, fill=gray!12, text=gray!90!black
    },
    arr/.style={-{Stealth[length=5pt]}, thick, gray!60!black},
    arrlabel/.style={font=\scriptsize, text=gray!80!black, align=center, fill=white,
        inner sep=1.5pt}
]
    \node[phase=blue] (abd) {Abduction};
    \node[phase=green!70!black, right=1.4cm of abd] (ded) {Deduction};
    \node[phase=orange, right=1.4cm of ded] (ind) {Induction};

    \node[layer=blue, below=0.3cm of abd] (L0) {L0};
    \node[layer=green!70!black, below=0.3cm of ded] (L1) {L1};
    \node[layer=orange, below=0.3cm of ind] (L2) {L2};

    \node[font=\scriptsize, text=gray, below=0.08cm of L0] {conjecture};
    \node[font=\scriptsize, text=gray, below=0.08cm of L1] {substantiated};
    \node[font=\scriptsize, text=gray, below=0.08cm of L2] {corroborated};

    \draw[arr] (abd) -- node[arrlabel, above]
        {logical\\[-1pt]consistency} (ded);
    \draw[arr] (ded) -- node[arrlabel, above]
        {empirical\\[-1pt]validation} (ind);

    \node[drr, right=1.2cm of ind] (drr) {DRR};
    \draw[arr] (ind) -- node[arrlabel, above]
        {decision\\[-1pt]final.} (drr);

    \draw[arr, dashed, red!60!black]
        (drr.south) -- ++(0,-1.4)
        -| (abd.south);
    \node[arrlabel, fill=white]
        at ($(abd.south)!0.5!(ded.south) + (0,-1.9)$)
        {anomaly / evidence decay};

\end{tikzpicture}
\caption{The ADI reasoning cycle. Abduction generates conjectures (L0),
    Deduction verifies logical consistency (L1), and Induction validates
    empirically (L2). Finalized decisions become Design Rationale Records
    (DRRs). Evidence decay or new anomalies re-enter the cycle.}
\label{fig:adi-cycle}
\end{figure}

\subsection{Design Rationale Records and the Reasoning Audit Trail}%
\label{sec:drr}

After completing the ADI cycle, a finalized decision is recorded as a
\emph{Design Rationale Record} (DRR): a structured decision record
augmented with the evidence chain that produced it. Each DRR captures:
\begin{enumerate}
    \item The \textbf{inference mode} at each step (abduction, deduction,
        or induction) and the epistemic layer of the resulting claim.
    \item The \textbf{evidence} supporting each transition, with explicit
        provenance---what was checked, against what constraints, with what
        outcome.
    \item The \textbf{reliability score} of the final claim, computed via
        the weakest-link aggregation principle described in
        Section~\ref{sec:wlnk}: the reliability of a conclusion equals
        the reliability of its least reliable supporting premise.
    \item The \textbf{scope specification} bounding the conditions under
        which the claim is valid, and a \textbf{validity window}
        specifying when the evidence must be re-evaluated.
\end{enumerate}

This audit trail directly addresses two concerns identified in the call
for papers. First, it provides a mechanism for \emph{detecting and
eliminating logical contradictions} across multiple reasoning steps
(Topic~3): because every claim is labeled with its inference mode and
epistemic layer, contradictions between claims at different layers are
structurally detectable rather than hidden in natural language traces.
Second, it provides a \emph{benchmark-amenable record} of the reasoning
process (Topic~5): the audit trail can be evaluated for logical
correctness, evidential completeness, and scope consistency independently
of the domain-specific content.

\paragraph{The Transformer Mandate.}
We introduce a structural constraint on the ADI cycle: the entity that
finalizes a decision (produces the DRR) must be external to the generation
loop. An LLM may propose hypotheses (abduction) and gather supporting
evidence (induction), but ratification---the transition from ``candidate
conclusion'' to ``accepted decision''---requires external verification.
This prevents a failure mode where an autonomous system bootstraps
confidence in its own outputs by citing its own prior
generations~\citep{ferrario2026epistemology}. The mandate is architectural,
not a policy preference: it is enforced by the protocol structure, which
requires the ratifying entity to have access to constraints independent of
the generation process.

\section{Algebraic Consistency Invariants}\label{sec:invariants}

\subsection{The Gamma Quintet}\label{sec:quintet}

We specify five algebraic invariants that any consistency-preserving inference operator $\Gamma: \mathcal{P}([0,1]) \to [0,1]$ must satisfy. These invariants constrain how reliability scores compose across multi-step reasoning, regardless of evaluation order, evidence arrangement, or chain length.

\begin{enumerate}
\item \textbf{IDEM} (Idempotence): $\Gamma([x]) = x$. A single premise retains its original reliability.
\item \textbf{COMM} (Commutativity): $\Gamma([a, b]) = \Gamma([b, a])$. The order in which premises are evaluated is logically irrelevant.
\item \textbf{LOC} (Locality): Changing premise~$E$ affects only conclusions whose derivation graph includes~$E$. Isolated subproofs remain stable.
\item \textbf{WLNK} (Weakest Link): $\Gamma(S) \leq \min(S)$. No inference chain can be more reliable than its least reliable premise.
\item \textbf{MONO} (Monotonicity): $a \leq a'$ implies $\Gamma([a, b]) \leq \Gamma([a', b])$. Strengthening a premise cannot weaken a conclusion.
\end{enumerate}

\noindent LOC is a system-level invariant governing how updates propagate through the derivation graph, distinct from the pointwise algebraic properties (IDEM, COMM, WLNK, MONO); we include it in the Quintet because the framework's consistency guarantees depend on it jointly with the operator properties.

\begin{theorem}[Quintet Satisfaction]\label{thm:quintet}
The G\"{o}del t-norm $\Gamma(S) = \min(S)$~\citep{hajek1998fuzzy} satisfies all five invariants.
\end{theorem}

\begin{proof}
IDEM: $\min([x]) = x$. COMM: $\min$ is symmetric over its arguments. LOC: $\min$ depends only on the multiset elements; modifying a value outside the multiset has no effect. WLNK: $\min(S) \leq \min(S)$ trivially. MONO: if $a \leq a'$, then $\min(a,b) \leq \min(a',b)$ by case analysis on $b$.
\end{proof}

\begin{theorem}[Idempotent Uniqueness~\citep{klement2000triangular,metcalfe2005fuzzy}]\label{thm:unique}
Among continuous t-norms on $[0,1]$, the G\"{o}del t-norm is the unique idempotent t-norm. This follows from the Klement--Mesiar--Pap characterization of continuous t-norms~\citep{klement2000triangular}: among continuous t-norms, those satisfying t-norm idempotence ($\Gamma(x,x) = x$ for all $x \in [0,1]$) are exactly those equal to $\min$. When we require operators to satisfy both IDEM (on singletons) and WLNK as an upper bound, restriction to the continuous t-norm family picks out $\min$ uniquely. The Quintet invariants are the contribution; $\min$ is one valid instantiation. A future learned aggregator satisfying the Quintet while permitting confidence accumulation for independent evidence would also be valid.
\end{theorem}

\begin{table}[ht]
\caption{Aggregation operators and Gamma Quintet compliance.}
\label{tab:alternatives}
\centering
\small
\begin{tabular}{@{}lccccl@{}}
\toprule
\textbf{Operator} & \textbf{IDEM} & \textbf{COMM} & \textbf{WLNK} & \textbf{MONO} & \textbf{Use Case} \\
\midrule
$\min$ (G\"{o}del) & \checkmark & \checkmark & \checkmark & \checkmark & Serial chains \\
Product & \checkmark & \checkmark & $\sim$ & \checkmark & Independent evidence \\
Mean & $\times$ & \checkmark & $\times$ & \checkmark & Not recommended \\
$\max$ & \checkmark & \checkmark & $\times$ & \checkmark & Not recommended \\
\bottomrule
\end{tabular}
\end{table}

Product aggregation satisfies WLNK when all scores are $\leq 1$ (since $a \cdot b \leq \min(a,b)$ for $a,b \in [0,1]$), making it a valid relaxation for genuinely independent evidence. Mean violates both IDEM and WLNK, which has direct consequences for logical consistency: three weak premises at $R = 0.4$ would average to $0.4$, masquerading as comparable to a single controlled experiment at $0.4$---a conflation that obscures the distinction between quantity and quality of evidence.

\subsection{WLNK as a Logical Consistency Rule}\label{sec:wlnk}

The WLNK invariant is a \emph{logical consistency constraint}, not a conservative aggregation heuristic. In a deductive argument chain $A \Rightarrow B \Rightarrow C$, the conclusion $C$ cannot be more reliable than the weakest premise in the chain. This is a structural property of valid inference: if any link in a derivation is uncertain, the derived conclusion inherits that uncertainty.

\textbf{Possibilistic logic foundation.} WLNK directly instantiates Dubois and Prade's weakest link principle from possibilistic logic: ``the strength of an inference chain is that of the least certain formula involved''~\citep{dubois1988possibility}. In possibilistic logic, each formula carries a necessity degree, and the resolution rule propagates the minimum degree through inference steps. Our WLNK invariant lifts this principle from propositional possibilistic logic to a general algebraic constraint on any consistency-preserving operator. Recent work~\citep{dubois2025possibilistic} extends these foundations to graded reasoning under uncertainty.

\textbf{As a logical inference rule.} We can state WLNK as an inference rule analogous to modus ponens:
\[
\dfrac{P_1 : r_1 \qquad P_2 : r_2 \qquad \cdots \qquad P_n : r_n \qquad P_1, \ldots, P_n \vdash C}{C : \min(r_1, \ldots, r_n)}
\]
This rule says: if a conclusion $C$ is derived from premises $P_1, \ldots, P_n$ with respective reliability scores $r_1, \ldots, r_n$, then $C$'s reliability is bounded by the minimum. Compare this to standard modus ponens, which propagates truth values; here we propagate \emph{graded} epistemic status through the same logical structure.

\textbf{Empirical validation for LLM reasoning.} Jacovi et al.~\citep{jacovi2024weakestlink} provide direct empirical evidence for WLNK in the context of chain-of-thought reasoning. They demonstrate that the reliability of multi-step LLM reasoning is bounded by the weakest individual step, confirming that WLNK is an empirically observable property of how reasoning chains degrade, not just a theoretical desideratum.

\textbf{Contradiction detection.} Consider an LLM system that answers question $Q_1$ based on premise $P$ with $R(P) = 0.9$, and later answers question $Q_2$ based on premise $P'$ that contradicts $P$, with $R(P') = 0.3$. Without WLNK, a system using arithmetic mean might assign both answers moderate reliability (e.g., $\frac{0.9 + 0.3}{2} = 0.6$), masking the contradiction. Under WLNK, the system tracks that any conclusion depending on both $P$ and $P'$ is capped at $\min(0.9, 0.3) = 0.3$---the reliability of the weakest (and contradicting) premise. The low aggregate signals that the knowledge base is inconsistent and requires resolution.

\textbf{Worked example.} Suppose an LLM constructs a three-step logical argument:
\begin{enumerate}
\item[S1.] ``Python's GIL prevents true parallelism'' ($R = 0.95$, well-established fact)
\item[S2.] ``Therefore, CPU-bound tasks cannot benefit from threading'' ($R = 0.85$, valid deduction from S1)
\item[S3.] ``Therefore, all Python programs should use multiprocessing'' ($R = 0.40$, overgeneralization---ignores I/O-bound workloads, async alternatives)
\end{enumerate}
Under WLNK: $R_{\text{chain}} = \min(0.95, 0.85, 0.40) = 0.40$. The weak final step correctly caps the entire argument. Under arithmetic mean: $R_{\text{chain}} = \frac{0.95 + 0.85 + 0.40}{3} = 0.73$---a score that would classify this argument as moderately reliable, hiding the logical overreach in S3. WLNK surfaces the single weak step rather than diluting it across the chain.

\textbf{Quadruple-triangulated justification.} The choice of $\min$ is supported by four independent lines: (1)~\emph{t-norm theory}---the G\"{o}del t-norm is the unique continuous idempotent t-norm~\citep{klement2000triangular,metcalfe2005fuzzy} (Theorem~\ref{thm:unique}); (2)~\emph{possibility theory}---Dubois and Prade's necessity-based inference propagates the minimum through deductive chains~\citep{dubois1988possibility,dubois2025possibilistic}; (3)~\emph{empirical measurement}---Jacovi et al.~\citep{jacovi2024weakestlink} confirm that chain-of-thought reliability is bounded by the weakest step; (4)~\emph{algebraic specification}---the Gamma Quintet derives WLNK from first principles (Theorem~\ref{thm:quintet}). Four independent lines arriving at the same operator is difficult to dismiss as coincidence.

\textbf{Two-tier evidence aggregation.} Applying flat $\min$ across heterogeneous evidence is epistemically unsound: a must-pass gate (``does the code compile?'') should aggregate differently from corroborating performance metrics. We extend to a two-tier architecture: \emph{Tier~1} classifies evidence by epistemic role (structural gates, performance metrics, quality assurance, etc.) and aggregates within each role using a role-appropriate operator (gates use $\min$; quality reviews use probabilistic sum; performance metrics use a conservative OWA). \emph{Tier~2} applies $\min$ across role-level scores, preserving WLNK: if any gate fails, the overall score is zero. This decomposition is WLNK-preserving because each within-role operator is a t-norm and the cross-role combiner is $\min$.

\section{The Framework as External Reasoning Scaffold}\label{sec:scaffold}

The framework is an external symbolic system alongside the LLM,
not a modification to the model's internal reasoning~\citep{gilda2026epistemic}. The LLM generates
natural-language hypotheses and arguments; the framework maintains a
symbolic knowledge graph that tracks epistemic status, dependency
structure, and formal invariants over those claims. The interface is
simple: the LLM proposes claims, and the framework records their
provenance, assigns epistemic layers, and enforces consistency constraints
across the growing knowledge graph.

This architecture targets four specific limitations of LLM
reasoning:

\paragraph{Conflated inference modes.}
LLMs routinely mix hypothesis generation with verification and evidence
gathering within a single response, producing outputs that appear
rigorous but conflate logically distinct operations. The ADI protocol
(Section~\ref{sec:adi}) forces explicit separation: abduction generates
candidate claims at L0, deduction checks logical consistency before
promoting to L1, and induction requires empirical evidence for promotion
to L2. Each mode has distinct preconditions and produces claims at a
specific epistemic layer, making the inference type of every claim
auditable.

\paragraph{Inconsistent reliability across responses.}
When an LLM generates multiple claims that depend on shared premises, it
assigns no consistent measure of confidence across the dependency graph.
The WLNK invariant (Section~\ref{sec:invariants}) enforces global
consistency: no composite claim can exceed the reliability of its weakest
supporting evidence, and this constraint propagates transitively through
the entire dependency structure.

\paragraph{No self-correction.}
The Transformer Mandate is an architectural constraint preventing
self-promotion loops: the agent that generates a claim cannot also
provide the evidence that promotes it. Layer promotion requires external
verification, enforced at the data model level rather than by prompt
instruction.

\paragraph{Stale knowledge.}
Every piece of evidence in the knowledge graph carries a
\texttt{valid\_until} timestamp. When evidence expires, dependent claims
are automatically flagged for re-validation. Validity periods are
formality-dependent: informal observations (F0) expire faster than
empirical measurements (F2), reflecting the intuition that rigorously
gathered evidence remains relevant longer.

\subsection{Comparison with Existing Approaches}\label{sec:comparison}

Tool-augmented LLMs extend computational \emph{capabilities} but impose no structure on the \emph{reasoning process}. Neuro-symbolic systems such as Logical Neural Networks~\citep{riegel2020lnn} and DeepProbLog~\citep{manhaeve2021deepproblog} provide strong guarantees but require full domain formalization. Our approach occupies a middle ground: lightweight symbolic structure---epistemic layers, dependency tracking, algebraic invariants---over natural-language claims, without requiring translation into formal logic. The framework tracks epistemic status and structural relationships rather than semantic content, so it applies to any domain without requiring domain-specific formalization.

\subsection{Architectural Pattern}\label{sec:roles}

The three components map to the ADI phases: (1)~\textbf{LLM as hypothesis generator} (Abduction)---proposes candidate explanations as L0 conjectures; (2)~\textbf{Framework as consistency checker} (Deduction)---verifies logical consistency against existing knowledge and WLNK before promoting to L1; (3)~\textbf{Empirical tools as validators} (Induction)---test runners, benchmarks, and human review provide evidence for L2 promotion. No single component controls the full epistemic lifecycle of a claim.

\section{Property-Based Verification Benchmark}\label{sec:pbt}

\subsection{Specification-Verification Approach}\label{sec:spec-verify}

Algebraic specification (Section~\ref{sec:invariants}) establishes correctness-by-construction; property-based testing~\citep{claessen2000quickcheck,goldstein2024pbt} establishes correctness-in-implementation, verifying that the actual code preserves invariants despite floating-point arithmetic and concurrent access. Together they constitute a verified reference implementation whose test suite can serve as a starting point for future reasoning benchmarks.

We verify 100 property-based tests across five specification areas, plus 16 fuzz tests, each exercising $10^5$+ randomly generated cases (Table~\ref{tab:pbt}). The largest group (57~tests) targets the $R_{\text{eff}}$ calculator: bounds/WLNK enforcement, ceiling caps, monotonicity, dependency propagation, two-tier evidence aggregation, and preset inheritance. Scope lattice tests (16) verify bounded lattice axioms and parse--serialize round-trips. FSM tests (11) verify phase ordering, reachability, and role enforcement. Graph topology tests (6) verify WLNK propagation through deep chains, diamonds, and mixed topologies. Inspector tests (10) verify BFS traversal correctness. All tests follow the industrial-scale PBT methodology of Arts et al.~\citep{arts2006testing}.

\subsection{PBT as Design Exploration}\label{sec:design-exploration}

Beyond verification, PBT surfaces implicit assumptions. For example, the property ``every phase can reach \textsc{idle}'' initially failed for \textsc{operation} because no back-transition was defined---not a bug, but an undocumented design assumption. The resolution was to classify \textsc{operation} as intentionally terminal and exclude terminal states from the property. This illustrates PBT's role as a consistency auditor: the invariants themselves must be internally consistent before they can meaningfully constrain external claims.

\subsection{Property Inventory}\label{sec:inventory}

\begin{table}[ht]
\caption{Property-based verification inventory for logical consistency
invariants. All 100~properties are randomized checks over $10^5$+
cases; all 16~fuzz tests use corpus-guided mutation~\citep{gofuzz2022}.}
\label{tab:pbt}
\centering
\small
\begin{tabular}{@{}llr@{}}
\toprule
\textbf{Specification Area} & \textbf{Verified Properties} & \textbf{Count} \\
\midrule
$R_{\text{eff}}$ calculator & Bounds, WLNK, ceilings, monotonicity, presets, two-tier & 57 \\
Scope algebra & Lattice axioms, match, round-trip & 16 \\
Epistemic FSM & No skipping, reachability, determinism, ordering & 11 \\
Graph topology & Deep chains, diamonds, mixed serial/parallel & 6 \\
Dependency inspector & BFS correctness, deduplication, layer preservation & 10 \\
Fuzz tests & IEEE 754 boundaries, parser round-trips, concurrency & 16 \\
\midrule
\textbf{Total} & & \textbf{116} \\
\bottomrule
\end{tabular}
\end{table}

\section{Related Work and Limitations}\label{sec:related}

\subsection{Related Work}

\textbf{LLM reasoning approaches.}
Chain-of-thought~\citep{wei2022chainofthought}, self-consistency~\citep{wang2023selfconsistency}, and process reward models~\citep{lightman2024verify} improve output quality but lack formal invariants guaranteeing algebraic properties such as monotonicity or weakest-link propagation. Our framework provides externally verifiable constraints on reasoning chain integrity. Recent benchmarks reveal a ``curse of complexity'' where accuracy degrades sharply with problem scale~\citep{zebralogic2025}; our framework externalizes structural constraints that are enforced symbolically and therefore do not degrade with problem size.

\textbf{Weakest link in reasoning chains.}
\citet{jacovi2024weakestlink} empirically demonstrate that chain-of-thought
reliability is bounded by its weakest step. Our framework formalizes this
as an algebraic invariant ($R_{\text{eff}} \leq \min_i R_i$) enforced by
construction and verified via property-based testing across $10^5$+
generated inputs, transforming an empirical observation into a provable
structural guarantee.

\textbf{Possibilistic logic.}
Possibilistic logic~\citep{dubois1988possibility,dubois2025possibilistic}
provides the theoretical foundation for our weakest-link aggregation.
Possibilistic inference follows the ``weakest link resolution'' rule: the
necessity of a derived conclusion equals the minimum necessity of formulas
in the derivation chain~\citep{dubois2025possibilistic}. Our WLNK bound
is a direct application of this principle, where the ADI cycle produces
claims whose reliability propagates according to possibilistic semantics.

\textbf{Argumentation frameworks.}
Abstract~\citep{dung1995acceptability} and structured argumentation~\citep{besnard2008elements} establish semantics for argument acceptability. Our framework shares the graph-of-claims structure but tracks \emph{reliability}, \emph{temporal validity}, and \emph{scope} rather than resolving attack relations.

\textbf{Neuro-symbolic reasoning.}
Logical Neural Networks~\citep{riegel2020lnn} and DeepProbLog~\citep{manhaeve2021deepproblog} require full domain formalization. Our framework occupies a lighter-weight middle ground: symbolic structure without complete logical formalization, applicable as an external scaffold for general-purpose LLM reasoning.

\subsection{Limitations}

\textbf{No machine-checked proofs.}
Our verification relies on property-based testing, not theorem provers (Coq~\citep{coq_team}) or model checkers (TLA+~\citep{lamport2002tla}). PBT provides high confidence through volume ($10^5$+ cases per invariant) but not exhaustive guarantees.

\textbf{End-to-end LLM evaluation is preliminary.}
An evaluation harness integrating the framework with GPT-4o on ML engineering tasks (AIRS-Bench~\citep{airsbench2026}) has been developed, with preliminary results showing reduced execution errors under framework-guided reasoning. However, evaluation on dedicated logical reasoning benchmarks such as ZebraLogic~\citep{zebralogic2025} and FOLIO~\citep{han2022folio} remains future work. Controlled experiments with the current codebase are ongoing.

\textbf{Ceiling values are policy defaults.}
Specific ceiling percentages (e.g., $C_{F_0} = 70\%$, $C_{L_0} = 35\%$)
are configurable defaults, not empirically calibrated. The implementation supports per-context configuration overrides, allowing domain-specific calibration of ceilings, congruence penalties, and evidence age thresholds. The \emph{ordering
invariant} ($C_{F_0} < C_{F_1} < C_{F_2} < C_{F_3}$) is verified by PBT
regardless of specific values, but empirical calibration against domain-specific outcomes remains future work.

\textbf{ADI requires structured interaction.}
Single-pass inference cannot use the ADI protocol; it requires multi-turn interaction or a multi-agent architecture---a structural overhead that is the cost of formal guarantees.

\textbf{Open questions:} (1)~Can WLNK serve as a differentiable training constraint? (2)~Can the ADI cycle be realized as a multi-agent protocol with specialized agents per inference mode? Preliminary work on programmatic tool calling---where an LLM directly queries the knowledge graph via structured function calls---suggests viability. (3)~Can epistemic and aleatoric uncertainty be decomposed within $R_{\text{eff}}$~\citep{hullermeier2021uncertainty}?

\section{Conclusion}\label{sec:conclusion}

We presented a symbolic reasoning scaffold that operationalizes Peirce's tripartite inference as an explicit ADI protocol for LLM-assisted reasoning, with five algebraic invariants (the Gamma Quintet) formally constraining reliability propagation. The central constraint---$\min$ as the unique idempotent continuous t-norm---converges from four independent lines: algebraic specification, possibilistic logic, empirical CoT measurement, and t-norm theory. A verification suite of 100 property tests and 16 fuzz tests validates implementation fidelity across $10^5$+ cases. We invite the community to extend this work with end-to-end evaluation on logical reasoning benchmarks, differentiable WLNK training objectives, and multi-agent ADI implementations.

\bibliography{references}
\bibliographystyle{iclr2026_conference}

\end{document}